\documentclass[conference]{IEEEtran}

\usepackage{times}

\usepackage[numbers]{natbib}
\usepackage{multicol}
\usepackage[bookmarks=true]{hyperref}

\usepackage[parfill]{parskip}

\usepackage{amsmath, amsthm}
\usepackage{amsfonts}
\usepackage{soul}
\usepackage{listings}

\usepackage[frozencache=true,cachedir=minted-cache]{minted}

\usepackage[most]{tcolorbox}
\usepackage{booktabs}

\newtheorem{thm}{Theorem}

\newtheorem{assumption}{Assumption}

\newtheorem{problem}{Problem}

\usepackage{multirow}

\usepackage{bibunits}

\usepackage{hyperref}

\newcommand\blfootnote[1]{%
  \begingroup
  \renewcommand\thefootnote{}\footnotetext{#1}%
  \addtocounter{footnote}{-1}%
  \endgroup
}

\begin{document}

% paper title
\title{Contextual Safety Reasoning and Grounding\\ for Open-World Robots
\vspace{-12pt}
}

\author{Zachary Ravichandran$^1$, David Snyder$^1$, Alexander Robey$^2$, Hamed Hassani$^1$, \\ Vijay Kumar$^1$, and George J. Pappas$^1$\\
$^1$University of Pennsylvania, $^2$Carnegie Mellon University 
}

\maketitle

\blfootnote{Correspondance to \texttt{zacravi@seas.upenn.edu}. This work was supported by the DARPA SAFRON grant, under award number HR0011-25-3-0135, the Distributed and Collaborative Intelligent Systems and Technology (DCIST) Collaborative Research Alliance (ARL DCIST CRA W911NF-17-2-0181), Coefficient Giving.  D. Snyder acknowledges the support of the Penn AI Fellowship and Z. Ravichandran acknowledges the support of the NSF Graduate Research Fellowship.}

\begin{abstract}
Robots are increasingly operating in open-world environments where safe behavior depends on context: the same hallway may require different navigation strategies when crowded versus empty, or during an emergency versus normal operations. Traditional safety approaches enforce fixed constraints in user-specified contexts, limiting their ability to handle the open-ended contextual variability of real-world deployment. 
We address this gap via \texttt{CORE}, a safety framework that enables online contextual  reasoning, grounding, and enforcement without prior knowledge of the environment (e.g., maps or safety specifications).
\texttt{CORE} uses a vision-language model (VLM) to continuously reason about context-dependent safety rules directly from visual observations, grounds these rules in the physical environment, and enforces the resulting spatially-defined safe sets via control barrier functions.
We provide probabilistic safety guarantees for \texttt{CORE} that account for perceptual uncertainty, and we demonstrate through simulation and real-world experiments that \texttt{CORE} enforces contextually appropriate behavior in unseen environments, significantly outperforming prior semantic safety methods that lack online contextual reasoning. Ablation studies validate our theoretical guarantees and underscore the importance of both VLM-based reasoning and spatial grounding for enforcing contextual safety in novel settings.
We provide additional resources at \url{https://zacravichandran.github.io/CORE}.

\end{abstract}

\IEEEpeerreviewmaketitle

\section{Introduction}
\label{sec:intro}

% TODO: complexity and diversity in settings and tasks 
As robots perform diverse tasks in complex settings such as homes, offices, and warehouses, notions of safety become increasingly contextual and semantic.
For example, a wet floor sign is not simply an obstacle to be avoided but denotes that the surrounding area may be dangerous; safety compliance requires that a robot understand this cue and act accordingly.
The open-world nature of these settings exacerbates contextual safety challenges. 
A robot is likely to encounter contexts that are unknown \emph{a priori},
so a system designer cannot simply provide a set of precise safety rules upfront.
Rather, the robot must \emph{reason about contexts} encountered in the wild.

Robot safety has typically been operationalized via metric safe sets and specification languages (e.g., signal temporal logic and linear temporal logic~\cite{MalerN04stl, Pnueli77ltl}).
Given such specifications, formalisms such as Control Barrier Functions (CBFs) or Hamilton-Jacobi Reachability (HJB) provably enforce safety~\cite{bansal_hamilton-jacobi_2017, ames_control_2019,robey2020learning, stdcbflars}.
However, safety specifications must be defined by an expert user and require complete \emph{a priori} knowledge on a robot's environment, making them impractical for widespread use in complex open-world settings. 
While recent work addresses operation in partially-known or dynamic environments, these works focus on geometric rather than contextual notions of safety~\cite{das2025robust,cosner2021measurement, yang2025shield}.

An emerging line of research seeks to augment traditional notions of safety with contextual and semantic information. 
These works enable users to describe safety rules in natural language (for example, ``stay away from the kitchen'')~\cite{feng2025wordssafetylanguageconditionedsafety,language_safety_hjb,semantic_manipulation}.
These language instructions are then grounded into some representation---such as a costmap---to define safe sets, and these sets are enforced with traditional safety tools such as HJB.
While these works take a valuable step towards enforcing contextual safety, they still require a user with knowledge of the operational environment to provide safety rules that reference precise semantics in that environment; the robot cannot reason about safety for itself.

\begin{figure}[t!]
    \centering
    \includegraphics[width=0.9\linewidth]{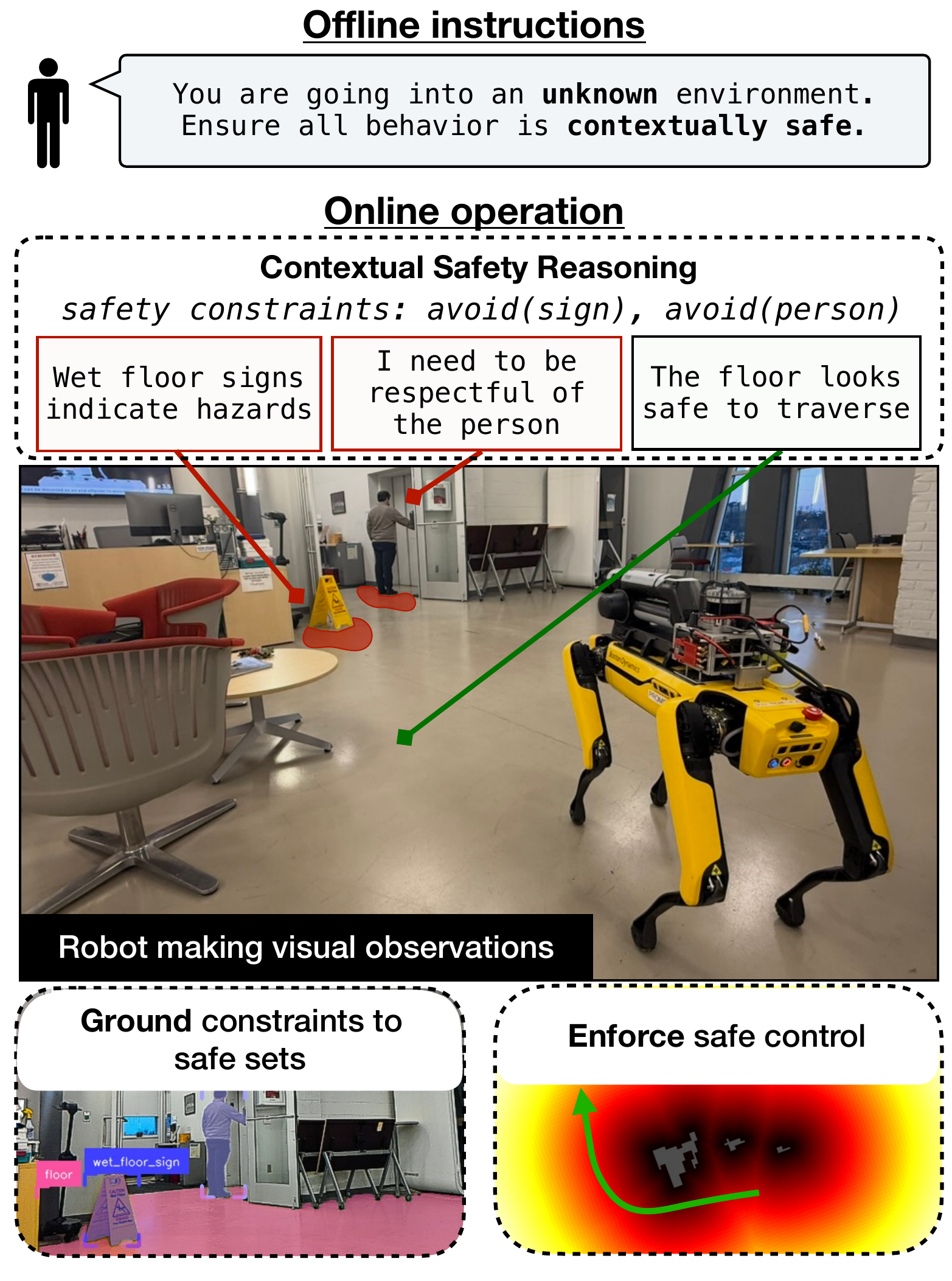}
    \caption{\texttt{CORE} enforces contextual safety via a three-stage process of contextual safety reasoning, semantic grounding, and safe control synthesis, enabling it to operate in open-world environments.}
    \label{fig:intro}
    \vspace{-12pt}
\end{figure}

% TODO: Add Observations I_t, safety constraints s_k etc., 

We address these limitations via \texttt{CORE}: Contextual Safety Reasoning and Enforcement,
a safety framework that operates without \emph{a priori} knowledge of the environment or safety specifications, as illustrated in Fig.~\ref{fig:intro}. 
\texttt{CORE} uses a VLM to reason about context-dependent safety constraints given visual observations from the robot.
It then grounds the inferred constraints in the robot's physical environment to produce spatially-defined safe sets.
While these safe sets may be enforced with a variety of methods, \texttt{CORE} uses CBFs due to their minimal requirements---a dynamics model, a barrier function, and a nominal control input.
Unlike traditional CBFs that assume a known and fixed barrier function, \texttt{CORE} constructs barriers online from noisy perception (due to the VLM inference and semantic grounding) and provides probabilistic safety guarantees despite this uncertainty.
To summarize, this paper makes the following \textbf{contributions}: 
\begin{enumerate} 
    \item [1.] A VLM-based reasoning module that infers contextual safety constraints from visual observations without \emph{a priori} environmental knowledge. \item [2.] A grounding method that maps natural language safety constraints to spatially-defined safe sets. 
    \item [3.] A CBF-based enforcement mechanism with probabilistic safety guarantees under perception uncertainty. 
\end{enumerate} 
We validate these contributions through simulation and hardware experiments, demonstrating that \texttt{CORE} rivals the behavior of an oracle using ground-truth context while preventing unsafe plans nearly five times more than baselines lacking contextual reasoning.
% TODO: quantitative 
The rest of the paper is structured as follows.
\S\ref{sec:related_work} overviews related work,
\S\ref{sec:formulation} provides a problem formulation, \S\ref{sec:method} describes the \texttt{CORE} architecture,
\S\ref{sec:experiments} reports experimental results,
\S\ref{sec:limitations} discusses \texttt{CORE}'s limitations and avenues for future work, and \S\ref{sec:conclusion} concludes.
\section{Related Work}
\label{sec:related_work}

\subsection{Robot Safety Filters}
\label{prior_work_on_control_filters}
Previous work considers CBFs with uncertainty in the state estimate~\cite{cosner2021measurement,dean2021guaranteeing, nanayakkara2025safety} or dynamics model~\cite{das2025robust, das2023robust, cosner2023generativemodelingresidualsrealtime, cosner2023robustsafetystochasticuncertainty}. 
Often, the uncertainty is modeled as a bounded disturbance~\cite{cosner2021measurement, dean2021guaranteeing, nanayakkara2025safety, das2023robust, das2025robust};  this admits robust safety guarantees conditioned on proper specification of the disturbance bound.
Alternatively, some recent works consider unbounded, stochastic uncertainty in (discrete time) system dynamics \cite{cosner2023generativemodelingresidualsrealtime, cosner2023robustsafetystochasticuncertainty}, which yield probabilistic guarantees in the form of $K$-step exit probabilities (i.e., the probability that the system will leave the safe set in $K$ steps).  By contrast, \texttt{CORE} models the uncertainty arising from perception-based barrier construction rather than from state or dynamics estimation.

Several lines of work have considered safety filters in the more general motion planning regime. Hamilton-Jacobi (HJ) Reachability-based methods~\cite{bansal_hamilton-jacobi_2017} construct safety guarantees using an explicit---and often, global---model of worst-case realizations of uncertainty~\cite{chen_fastrack_2021, bajcsy_2019_hj_navigation, Margellos_HJ_multiagent_2009}. These methods suffer from the curse of dimensionality, though efficient approximations exist~\cite{chen_exact_2017, chen_speedup_2016, herbert_hj_speedup_2021}.
Recent efforts extend these methods to encompass and inform settings with  high-dimensional perception and state representations \cite{bansal_deepreach_2021, feng_bridging_2025, nakamura_generalizing_2025}.
However, the computational overhead of these methods makes online implementations challenging. 
Other motion planning methods have built upon model-predictive control to generate near-optimal motion plans subject to robust measures of safety~\cite{Gandhi_robust_MPPI_2021, Yin_ShieldMPPI_2023}, but they strongly couple safe control and planning.
By contrast, \texttt{CORE} is agnostic to the robot's planner and controller.

\subsection{Contextual Robot Safety}

An emerging line of work uses language and semantics to specify safety for robotic control \cite{ahn_autort_2024, kundu_specific_2023, varley_embodied_2024, sinha_real-time_2024, elhafsi_semantic_2023, sermanet_generating_2025}. Most commonly, these approaches assume the semantic context is given and focus on enforcing such context in policy behavior. For example, \cite{feng2025wordssafetylanguageconditionedsafety} presents a planner that respects language-specified safety instructions (e.g., ``avoid the swimming pool''). 
\cite{language_safety_hjb} considers a similar problem formulation which promotes safety via reachability-based planners. More recently, several works have abstracted these constraints via latent representations \cite{wu2025forewarn}, user-described scenes \cite{semantic_manipulation}, or from more general multimodal data \cite{sermanet_generating_2025}. In contrast, \texttt{CORE} does not require context-specific safety constraints to be provided by a user \emph{a priori}; as illustrated in Fig. \ref{fig:intro}, these safety constraints are inferred and enforced online. 

A parallel body of research enforces contextual safety at the planning level~\cite{yang2023plugsafetychipenforcing, ravichandran_roboguard}.
These works receive  a world model---such as a scene graph---and provide safety constraints via linear temporal logic or similar formalisms designed to verify discrete actions sequences~\cite{quartey2025verifiablyfollowingcomplexrobot, chen2023nl2tl, optimalscenegraphllm}.
While many of these methods evaluate against common-use safety scenarios, a line of research considers adversarial attacks~\cite{robey2024jailbreaking, jones2025adversarialattacksroboticvision}.
Similar to the previously discussed control approaches, many of these works assume that the relevant context is provided \emph{a priori} via user instructions or an intermediate representation~\cite{yang2023plugsafetychipenforcing, quartey2025verifiablyfollowingcomplexrobot}.
\texttt{CORE} extends these results by reasoning about and enforcing contextual safety directly from visual observations---removing the need for an intermediate representation---and enforcing safety at the continuous control level.

\subsection{VLMs for Robot Context Analysis}

The robotics community has leveraged VLMs for robotic contextual reasoning in applications including 
traversability estimation~\cite{elnoor2024robotnavigationusingphysically}, language-driven trajectory generation~\cite{sathyamoorthy2024convoicontextawarenavigationusing, windecker2025navitraceevaluatingembodiednavigation, song2024vlmsocialnav, indoorvlmnav2025}, and object search~\cite{xie2023reasoningvlmnav}. 
Grounding VLM reasoning in the physical world is a key challenge, which
recent work addresses this by imposing some structure on the VLM's input image---such as an overlayed grid of numbers or sample trajectories---from which the VLM selects~\cite{sathyamoorthy2024convoicontextawarenavigationusing, elnoor2024robotnavigationusingphysically, indoorvlmnav2025}. 
In such approaches, the quality of the VLM's prediction is limited by the resolution of the imposed structure (e.g., the density of sampled trajectories), which can often be coarse. 
In contrast, \texttt{CORE} infers safety constraints via a representation that accounts for the unique context within an image.

\setcounter{footnote}{1}

\begin{figure*}[t!]
    \centering
    \includegraphics[width=0.90\linewidth]{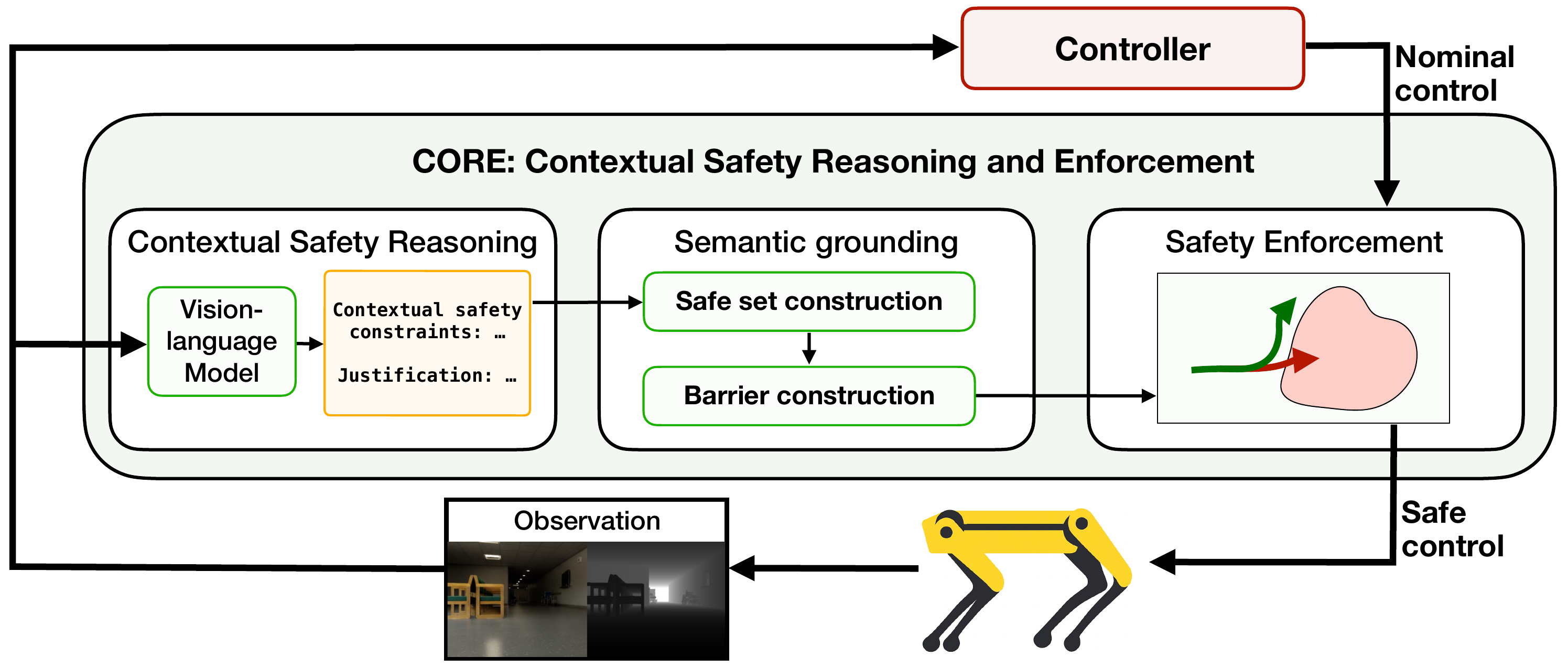}
    \caption{The \texttt{CORE} Framework enforces contextual safety via three modules. First, a VLM provides context-dependent safety constraints from visual observations. Constraints are grounded into spatially-defined safe sets, which are used for CBF-based control synthesis.}
    \label{fig:method}
\end{figure*}

\section{Problem formulation}
\label{sec:formulation}
We consider a mobile robot operating in an unknown environment where safety depends on semantic context that cannot be specified \emph{a priori}. 
The robot is equipped with an RGB-D camera and odometry, and has control-affine dynamics:
\begin{equation}
\label{eq:dynamics}
    \dot{x} = f(x) + g(x) u,
\end{equation}
where $x \in \mathbb{R}^n$ is the robot's state, $u \in \mathbb{R}^m$ is the control input, and $f: \mathbb{R}^n \rightarrow \mathbb{R}^n$ and $g: \mathbb{R}^n \rightarrow \mathbb{R}^{n \times m}$ are known.
A nominal controller provides a desired input $u_{\text{nom}}$, which may not satisfy safety constraints.
The robot is provided with high-level safety guidelines $\mathcal{G}$ (e.g., ``follow pedestrian social norms'') but must infer and enforce context-specific safety constraints from visual observations $I_t$ at runtime.

Achieving contextual safety in this setting requires addressing three interconnected challenges, as illustrated in Fig.~\ref{fig:intro}:
\emph{(i) Contextual reasoning}---inferring what constitutes safe behavior from visual observations;
\emph{(ii) Spatial grounding}---mapping the inferred contextual constraints to spatial safe sets $\mathcal{S} \subset \mathbb{R}^n$ that the robot can enforce, without access to pre-built maps; and
\emph{(iii) Probabilistic enforcement}---synthesizing safe control inputs despite uncertainty in perception, since both contextual inference and spatial grounding are imperfect.

\begin{problem}[Contextual Safety]
\label{prob:contextual_safety}
Given a robot with dynamics~\eqref{eq:dynamics}, high-level safety guidelines $\mathcal{G}$, and streaming visual observations $\{I_t\}$, synthesize a control input $u_{\text{safe}}$ such that:
\begin{enumerate}
    \item[(i)] \textbf{Reasoning:} Contextually appropriate safety constraints $\{s_1, \ldots, s_K\}$ are inferred from $I_t$ and $\mathcal{G}$,
    \item[(ii)] \textbf{Grounding:} Each constraint $s_k$ is grounded as a spatial safe set $\mathcal{S}_k \subset \mathbb{R}^n$, and
    \item[(iii)] \textbf{Safe control:} The robot remains within $\mathcal{S} = \bigcap_k \mathcal{S}_k$ with probability at least $1 - \delta$,
\end{enumerate}
while minimizing $\|u_{\text{safe}} - u_{\text{nom}}\|$.
\end{problem}

The following section presents \texttt{CORE}, a framework that addresses Problem~\ref{prob:contextual_safety} through three components: a VLM-based reasoning module for (i), a semantic grounding module for (ii), and a probabilistic CBF formulation for (iii).

\section{The \texttt{CORE} Architecture}
\label{sec:method}

The \texttt{CORE} architecture operates within the robot's control loop to infer, ground, and enforce contextually safe behavior via three modules, as illustrated in Fig.~\ref{fig:method}.
The first module uses a vision-language model (VLM) to infer context-specific safety constraints from visual observations.
These constraints are grounded into the robot's environment via a perception system that constructs spatially-defined safe sets.
Finally, these safe sets are used to synthesizes a control input which minimally modifies the nominal controller while satisfying contextualized and grounded safety constraints.

\subsection{Contextual Safety Reasoning}
\label{sec:safety_reasoning}

The contextual safety reasoning module infers safety constraints from visual observations using a vision-language model (VLM). 
Ensuring high-quality safety predictions requires a structured safety representation and VLM reasoning process, each of which we describe below.

\noindent \textbf{Structured safety representations.}  
To enable effective reasoning, we structure the VLM's output as a set of \emph{predicates}, which pairs semantic class (e.g., \verb|wet_floor_sign|, \verb|floor|) with a spatial operator (e.g., \verb|NEAR|, \verb|ON|) to define a region in the environment.
While \texttt{CORE} is agnostic to the specific set of operators, we consider four in our work: \verb|ON|, \verb|NEAR|, \verb|AROUND|, and \verb|BETWEEN|.
Each operator defines a safety logic outlining its definition and role in constructing safety predicates.\footnote{Please see the Appendix for further details.}
A predicate instantiates a semantic operator with a concrete semantic observed in the image. 
For example, a valid predicate for a robot operating in a building would be \verb|ON(floor)| or \verb|NEAR(desk)|.

\noindent \textbf{VLM reasoning process.} 
The VLM generates contextual safety constraints, $s = \{s_1, \dots, s_k\}$ via two sets of predicates: $\mathcal{P}_{\text{safe}}$ defining safe regions and $\mathcal{P}_{\text{unsafe}}$ defining unsafe regions.
The structure we place on the generation process is key to providing high-quality safety predictions. 
First, the VLM identifies safety-relevant classes such as hazard indicators or traversable terrain. 
The VLM then applies the safety logic defined by the available spatial operators to these semantics, in order to determine which operators are appropriate in the given context. 
For instance, the \verb|AROUND| operator indicates that the vicinity around an entity is dangerous.
When operating in an office, the VLM may observe that a \verb|wet_floor_sign| indicates surrounding danger, so it should construct the predicate \verb|AROUND(wet_floor_sign)|.
The VLM also provides a chain-of-thought justification of its safety decision~\cite{wei2022chain}.
The final output is a  JSON object containing the VLM's justification, safety-relevant semantics, and the predicates defining safe and unsafe regions, as illustrated in Fig.~\ref{fig:vlm_prediction}.B.
The set of operators, safety logic, and reasoning instructions are provided to the VLM via its system prompt.
As shown in \S\ref{sec:vlm_prediction}, this design greatly aids contextual safety reasoning.

\begin{figure}
    \centering
    \includegraphics[width=0.95\linewidth]{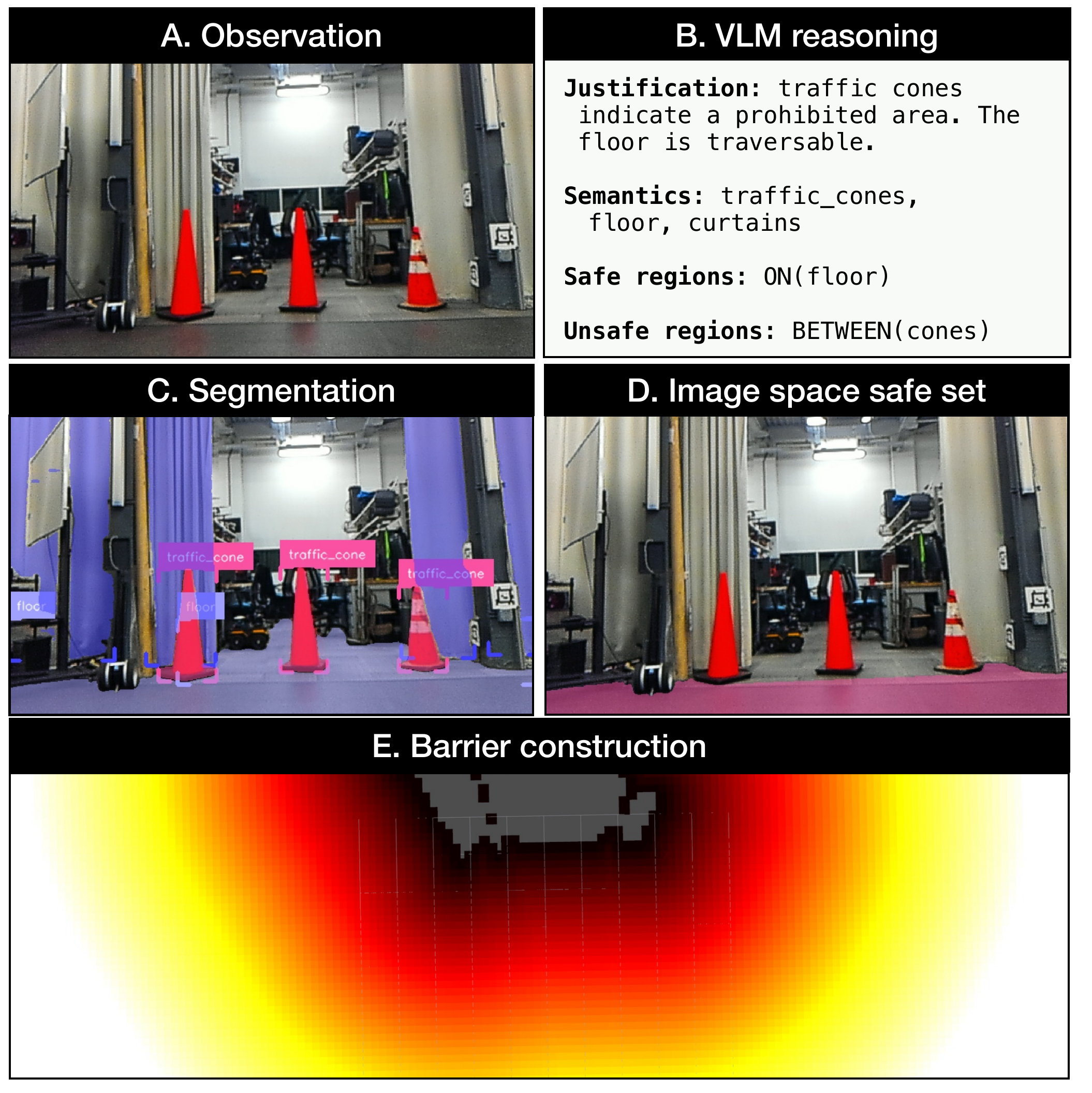}
    \caption{Illustration of \texttt{CORE}'s contextual reasoning and grounding process. Given an observation (A), the VLM predicts contextual safety constraints (B). \texttt{CORE} then segments the image (C) and constructs a image space safe set (D), which is integrated into a barrier function (E).}
    \label{fig:vlm_prediction}
    \vspace{-12pt}
\end{figure}

\subsection{Semantic Grounding}
This module grounds the VLM's safety constraints into spatially-defined safe sets that can be used for downstream control synthesis. 
Specifically, it constructs a barrier function, $h(x)$, which defines the safe set as all values of $x$ for which $h(x) \geq 0$. 
To construct this, the grounding module produces a safe set in the image space, projects it into the physical world, then integrates it into the barrier function.

\noindent \textbf{Image space safe set construction.}
The grounding module first localizes each semantic class identified by the VLM using an open-vocabulary segmentation model \cite{carion2025sam3segmentconcepts}.
It then constructs a pixel set for each predicate provided by the VLM (e.g., the predicate \verb|ON(floor)| would correspond to all pixels belonging to the \verb|floor| class in the segmented image).
Specifically, \verb|ON| and \verb|NEAR| selects pixels of the class, with \verb|ON| being reserved for traversable surfaces; \verb|AROUND| applies morphological dilation around the semantic instance; and \verb|BETWEEN| computes the convex hull of semantic instances.
This yields two collections of pixel sets corresponding to safe and unsafe predicates: $\{\mathcal{I}_j^{\text{safe}}\}$ and $\{\mathcal{I}_k^{\text{unsafe}}\}$.
The final safe set in image space is computed as:
\begin{equation}
    \mathcal{I}_{\text{safe}} = \left(\bigcup_j \mathcal{I}_j^{\text{safe}}\right) \setminus \left(\bigcup_k \mathcal{I}_k^{\text{unsafe}}\right),
\end{equation}
yielding a partition of safe and unsafe regions in the image. Fig.~\ref{fig:vlm_prediction}.C and D illustrate this process.

\noindent \textbf{Barrier function construction.}
The module then projects the image space safe and unsafe sets into a 3D point cloud using registered depth.
These points are accumulated into a 2D grid $\mathcal{C}$ which discretizes the robot's workspace.
For each cell $\mathcal{C}_{ij} \in \mathcal{C}$, we maintain counts $n_{ij}^{\text{safe}}$ and $n_{ij}^{\text{unsafe}}$ of safe and unsafe point observations falling within that cell.
We then associate a given robot state $x$ to its closest cell, $C_{ij}$, which provides the following estimate that a given state is safe:
\begin{equation}
    \mathbb{P}[x \text{ is safe}] = \frac{n_{ij}^{\text{safe}}}{n_{ij}^{\text{safe}} + n_{ij}^{\text{unsafe}}}.
\end{equation}

We threshold this probability to construct the the safe set $\mathcal{S} = \{x \mid \mathbb{P}[x \text{ is safe}] \geq \tau\}$, where $\tau$ is a hyperparameter.
Finally, we define the barrier function $h(x)$ as a signed distance field over $\mathcal{S}$:
\begin{equation}
\label{eq:barrier}
    h(x) = \begin{cases}
        \min_{y \in \partial \mathcal{S}} \|x - y\| & \text{if } x \in \mathcal{S} \quad \text{(safe)} \\
        -\min_{y \in \partial \mathcal{S}} \|x - y\| & \text{if } x \notin \mathcal{S} \quad \text{(unsafe)}
    \end{cases}
\end{equation}
where $\partial \mathcal{S}$ is the boundary of the safe set as computed from the thresholded grid.
This formulation ensures the desired properties of $h(x)$ as defined above and facilitates the CBF-based control synthesis approach described below.

\subsection{Contextual Safety Enforcement}
\label{sec:cbf}

The final module enforces contextual safety via control barrier functions (CBFs), which provide three key advantages for our setting: they do not require knowledge of the nominal controller (enabling compatibility across a variety of control approaches), they provide formal guarantees on safety adherence, and they admit efficient real-time solutions.
CBFs ensure safety by enforcing \emph{forward invariance}, meaning that if the robot starts in the safe set, it remains there:
\begin{equation}
\label{eq:invariance}
    x(0) \in \mathcal{S} \implies x(t) \in \mathcal{S} \quad \forall t \geq 0.
\end{equation}
For control-affine dynamics, forward invariance is guaranteed if the control input $u$ satisfies the constraint:
\begin{equation}
\label{eq:invariance}
    \langle \nabla h(x), f(x) + g(x) u \rangle  \geq -\alpha(h(x)),
\end{equation}
where $\alpha: \mathbb{R} \to \mathbb{R}$ is an extended class-$\mathcal{K}$ function (i.e., strictly increasing with $\alpha(0) = 0$)~\cite{ames_control_2019}.

Given the estimated barrier function $h(x)$ and a nominal control input $u_{\text{nom}}$, this module synthesizes a safe control input by solving the following optimization problem:
\begin{align}
\label{eq:control_synthesis}
    u_{\text{safe}} = \arg\min_u & \quad \|u - u_{\text{nom}}\|_2^2  \\
    \text{s.t.} & \quad \langle \nabla h(x), f(x) + g(x) u \rangle  \geq -\alpha(h(x)). \nonumber
\end{align}
The signed distance field representation of $h(x)$ in Eq.~\ref{eq:barrier} provides an analytical gradient in the form:
\begin{equation}
    \nabla \hat{h}(x) = \frac{x - y^*(x)}{\|x - y^*(x)\|}, \quad y^*(x) = \arg \min_{y \in \partial \mathcal{S}} \|x - y\|,
\end{equation}
where $y^*(x)$ is the closest point on the safe set boundary and can be found efficiently via grid lookups.
Eq.~\ref{eq:control_synthesis} is a quadratic program that can be solved efficiently using existing optimization tools; because the robot's dynamics are control affine and $u_\text{nom} = 0$ is a valid control input, a feasible solution will always exist.  
A challenge is that $h(x)$ evolves during robot operation, and we make the following assumption to ensure that the CBF invariance condition holds (Eq.~\ref{eq:invariance}):
\begin{assumption}[Safe initialization]
\label{assume:teleportation}
At any timestep $t$ when the robot is in the safe set, $h(x) \geq 0$,
updates to $h(x)$  cannot instantaneously classify the robot's current position as unsafe.
\end{assumption}
This assumption requires that contextually unsafe regions be identified by perception---i.e., VLM estimation and semantic grounding---before entry by the robot.
We now formally analyze this assumption by modeling the effect of perception uncertainty on CBF safety guarantees.

\subsection{Accounting for Perception Uncertainty}
\label{the_perception_uncertainty}
The high-level contextual safety of the \texttt{CORE} framework depends strongly on contextual understanding from perception. A crucial uncertainty arises in \emph{failed detections} which can severely impair safety in the downstream control. These considerations complement, but differ from, existing work on measurement-robust barrier functions (see \S\ref{sec:related_work}) that only consider disturbances in the state. 

The uncertainty module relies on the notion of a detection probability function $\underline{m}(\delta x): \mathbb{R}^3 \mapsto [0, 1]$, which expresses the (Bernoulli) probability of correctly detecting, from a new measurement, an unsafe region at relative position $\delta x$ with respect to the robot. Using this function, we give high-probability bounds on safe traversal rates across multiple contexts and environments with respect to failures in perception.

To define $\underline{m}(\cdot)$, the key technical concern is correlation in the observation data. That is, if an unsafe region is difficult to detect at time $t$, it is likely to be so at time $t+1$. This is addressed via a two-part decomposition. First, we design the guarantee to aggregate over contexts and environments. Thus, any worst-case contexts in which the VLM cannot make safe detections can be apportioned a separate risk fraction ($\gamma < \delta$) and can be ignored in the choice of $\underline{m}(\cdot)$; this is related to controlling a `value-at-risk' of the detection function \cite{majumdar_how_2017}. Second, we enforce conservatism in $\underline{m}(\cdot)$ to account for the temporal correlation. Adding conservatism after filtering out extreme, worst-case instances is designed to form a practical guarantee which reflects the downstream empirical performance -- this will be validated in Sect. \ref{sec:experiments}. Finally, we address a small notational mismatch: the control implementation is modeled in continuous time, whereas the perception is modeled in discrete time. This can be understood as, effectively, a simple two-part cascade, where the low-level control operates as a fast inner loop on the current perceptual representation (i.e., the current map), and the perceptual representation is itself updated in a slower outer loop. Given the practical perceptual latencies (see Tab. \ref{tab:safety_reasoning} and \ref{tab:vlm_ablation}), the loop timescales are separated by nearly two orders of magnitude, meaning they can be effectively decoupled and analyzed separately \cite{naidu_singular_2001}.

Conditioned on a design of $\underline{m}(\cdot)$, we analyze a proxy for the expected distance $r$ to the \emph{nearest undetected unsafe region}. Intuitively, if $r > 0$ then we currently remain safe. We define an inverse regularized distance function: 
\begin{equation}
\label{the_regularized_inverse_distance}
    d^{-1}(r; c, \ell) = \frac{c}{r + \ell}, 
\end{equation}
where $c, \ell > 0$ are positive hyperparameters. Intuitively, this function links the `dynamics' of gathering information to the expected physical safety margin; as more data is collected, the expected inverse distance shrinks, and the margin increases. From a fixed position, a robot taking $k$ measurements can upper bound the expectation of the inverse distance as: 
\begin{equation}
    \label{expected_inverse_distance_versus_r_equation}
        \mathbb{E}[d^{-1}(r; c, \ell)]\{k\} \leq \int_{0}^{\infty} \frac{c (1-\underline{m}(r))^k}{(r+\ell)Z}dr.
\end{equation}
Here $Z$ is a normalization constant to enforce the property of a valid probability distribution over the initial state of $r$. 
\begin{thm}[Probabilistic Trajectory-Length Safety]
\label{the_probabilistic_cbf_theorem}
    Assume a measurement function $\underline{m}(\cdot)$ satisfying the conditions given in Section \ref{the_perception_uncertainty}. Further, assume an initial safe radius $R > 0$ and consider a safety-filtered system which gathers, at each time $t$, $k^*(t)$ observations such that:
    \begin{equation}
        \label{the_theorem_equation}
        \mathbb{E}[d^{-1}(r; c, \ell)]\{k^*\} \leq \frac{c (\delta-\gamma)}{\ell}.
    \end{equation}
    Then, $\mathbb{P} \big[\exists t : x_t \notin \mathcal{S} \big] \leq \delta$. 
\end{thm}
\begin{proof}
    The full proof is deferred to the Appendix. Using results from nonnegative supermartingales and Ville's Inequality (similar in technique to \cite{cosner2023robustsafetystochasticuncertainty}), the expectation of $d^{-1}(\cdot)$ is precisely linked to its upper distribution quantiles, which allows for control of the probability of entering an undetected safe region at level $\delta - \gamma$. A union bound with the pre-specified worst-case environments (for which $\underline{m}(\cdot)$ is allowed to be invalid) gives the total risk at level $\delta$.
\end{proof}

\begin{figure}[t!]
    \centering
    \includegraphics[width=0.95\linewidth]{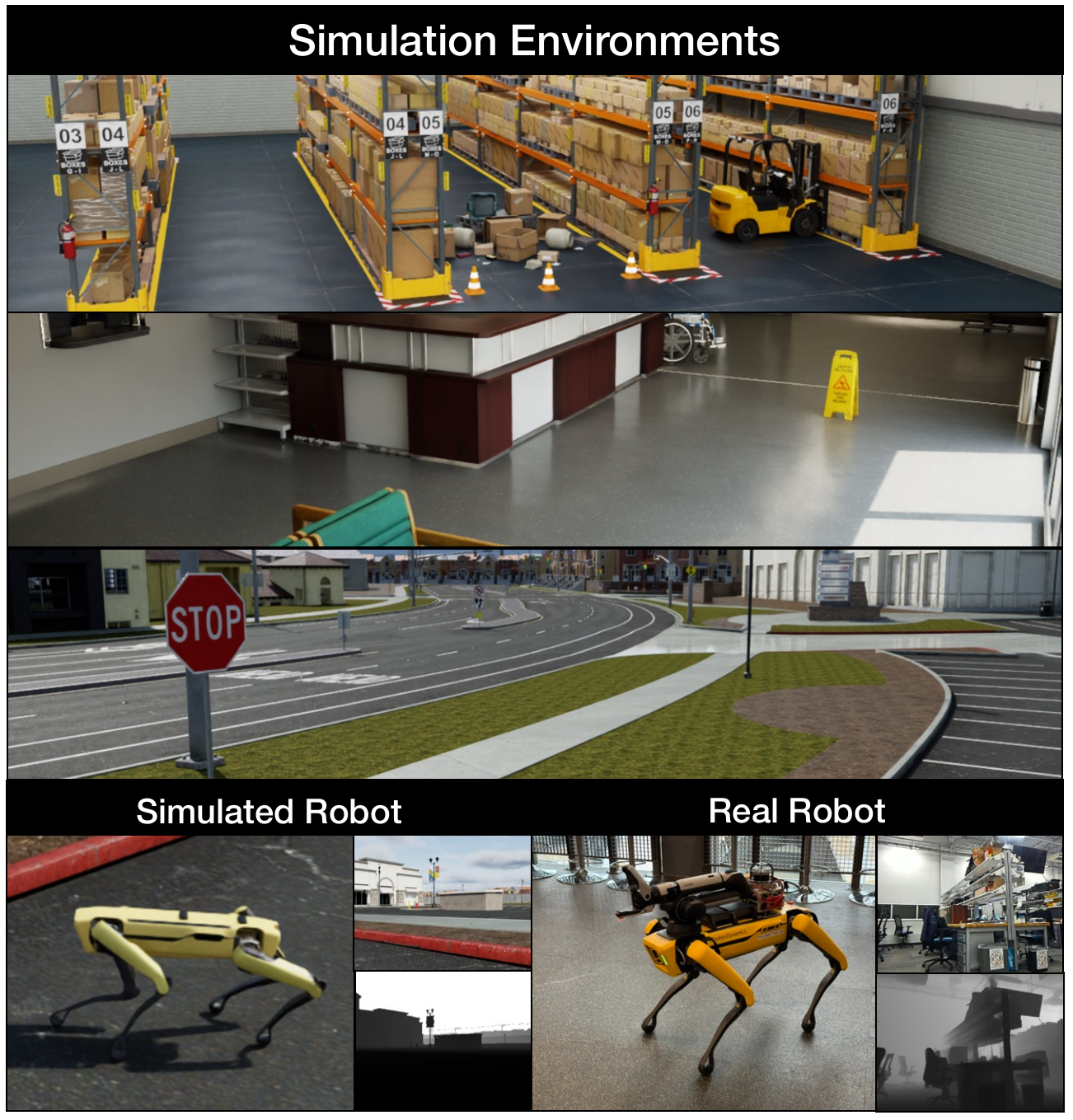}
    \caption{Simulation Environments and platforms used for experiments.}
    \label{fig:exp_setup}
    \vspace{-12pt}
\end{figure}

\section{Experimental Results}
\label{sec:experiments}

We evaluate \texttt{CORE} through simulation and real-world experiments to answer the following questions:
\begin{enumerate}
    \item [1.] Can \texttt{CORE} reason about contextual safety constraints without prior knowledge of the environment?
    \item [2.] Can \texttt{CORE} accurately ground the inferred contextual  constraints in the robot's physical environment?
    \item [3.] Can \texttt{CORE} enforce the contextual constraints while minimally interfering with safe nominal behavior?
\end{enumerate}

We first describe the key elements of our experimental setup: environments and platforms, safety scenarios, baselines, and implementation details. 
We then present quantitative results comparing \texttt{CORE} to baselines, ablation studies on optimal VLM models and effective contextual safety reasoning, real-world deployment results, and analysis of our probabilistic safety guarantees under perception uncertainty.

\noindent \textbf{Environments and platforms.} 
We first evaluate in NVIDIA Isaac Sim, a photorealistic physics simulator with high-fidelity robot dynamics~\cite{NVIDIA_Isaac_Sim}.
We consider three environments as illustrated in Fig.~\ref{fig:exp_setup}: a warehouse, a hospital, and a outdoor residential area.
% containing heavy machinery, hazard signs, and restricted zones;  a hospital with medical equipment, people, and patient areas; and a large-scale outdoor residential area featuring sidewalks, buildings, and pedestrian spaces.
These environments collectively provide a wide range of contextual safety scenarios, from industrial hazards to social norms.
We also deploy \texttt{CORE} on a Boston Dynamics Spot robot operating in a lab space with people, desks, hazard signs, and similar semantics.

For both simulation and real-world experiments, the robot is equipped with a forward-facing RGB-D camera and provides odometry.
We implement our system with the following dynamics model:
\begin{align}
    x_{t+1} = x_t + \Delta t \begin{bmatrix}
        \cos \theta_t & -\sin \theta_t & 0 \\
        \sin \theta_t & \cos \theta_t & 0 \\
        0 & 0 & 1
    \end{bmatrix}
    u_t.
\end{align}
where $x_t \in \mathbb{R}^3$ is the robot state comprising 2D position and heading, $u_t \in \mathbb{R}^3$ is the control input comprising 2D linear velocity and yaw rate, and $\Delta t = 0.1$s is the control period.
While our theoretical analysis in \S\ref{sec:method} uses continuous-time dynamics, we use a discrete-time implementation as is common in practice~\cite{Brunke2024PracticalCF, Taylor2022SafetySD}.

\noindent \textbf{Safety Scenarios}
We use the following contextual safety scenarios where obstacle avoidance is insufficient.
\begin{enumerate}
    \item [1.] \textbf{Contextual buffers} comprise objects that imply danger in their vicinity. For example, a forklift in a warehouse requires a safety buffer because heavy machinery may move unpredictably and cause damage.
    \item [2.] \textbf{Contextual barriers} comprise certain arrangements of objects that indicate prohibited regions. For example, a line of traffic cones signals a restricted area, even though individual cones are not physical obstacles.
    \item [3.] \textbf{Contextually-appropriate traversal} requires the robot must distinguish between appropriate and inappropriate navigation surfaces. For example, a sidewalk is appropriate for robot navigation while adjacent grass is not, despite both being physically traversable.
\end{enumerate}
Using these categories,
we design unsafe tasks where the nominal controller guides the robot toward violating the contextual safety constraints via adversarially placed goals.
We then design tasks where the nominal controller respects safety.
We evaluate four tasks per environment---two safe and two unsafe---and we repeat each task five times, yielding 60 evaluations. Please find further details in the Appendix.

\noindent \textbf{Baselines.}
We consider the following three baselines:
\begin{enumerate}
    \item [1.] \textbf{Oracle:} This baseline provides \texttt{CORE} with ground truth constraints \emph{a priori} via natural language instructions.
    \item [2.] \textbf{No Context:} This baseline removes the contextual reasoning ability of \texttt{CORE} via the following modification. Instead of using a VLM to produce safety rules online, this baseline uses an LLM---provided with the same system prompt outlining safety logic---to produce safety rules before deployment, similar to prior work that operates with upfront safety rules~\cite{feng2025wordssafetylanguageconditionedsafety,semantic_manipulation,language_safety_hjb}.
    \item [3.] \textbf{Geometric.} This safety filter which only performs obstacle avoidance based on metric information.
\end{enumerate}
To ensure a fair comparison, all baselines use the control synthesis approach outlined in \S\ref{sec:cbf}. 

\noindent \textbf{Implementation details.}
We implement the Safety Reasoning Module with a 4 bit quantized Gemma 3 27B VLM. 
The Semantic Grounding Module's segmentation is performed with SAM3~\cite{carion2025sam3segmentconcepts} and spatial operators are implemented with OpenCV~\cite{opencv_library}.
Because the Spot's RGB-D camera provides noisy and sparse depth estimates, real-world experiments predict smooth depth using DepthAnythingV3~\cite{depthanything3},  associate depth estimates with returns from the real camera, and estimate scaled depth using median fitting.
We implement the Semantic Grounding Module's costmap with a resolution of $0.2m$ and probability threshold of $0.5$.
The controller instantiates Eq.~\ref{eq:control_synthesis} with CVXPY~\cite{diamond2016cvxpy} using $\alpha(x) = 0.25 x$.
The nominal controller uses a waypoint follower with a PID controller.
Finally, we perform all closed-loop experiments in ROS2.
Please find further details in the Appendix.

\begin{figure}[t!]
    \centering
    \includegraphics[width=0.95\linewidth]{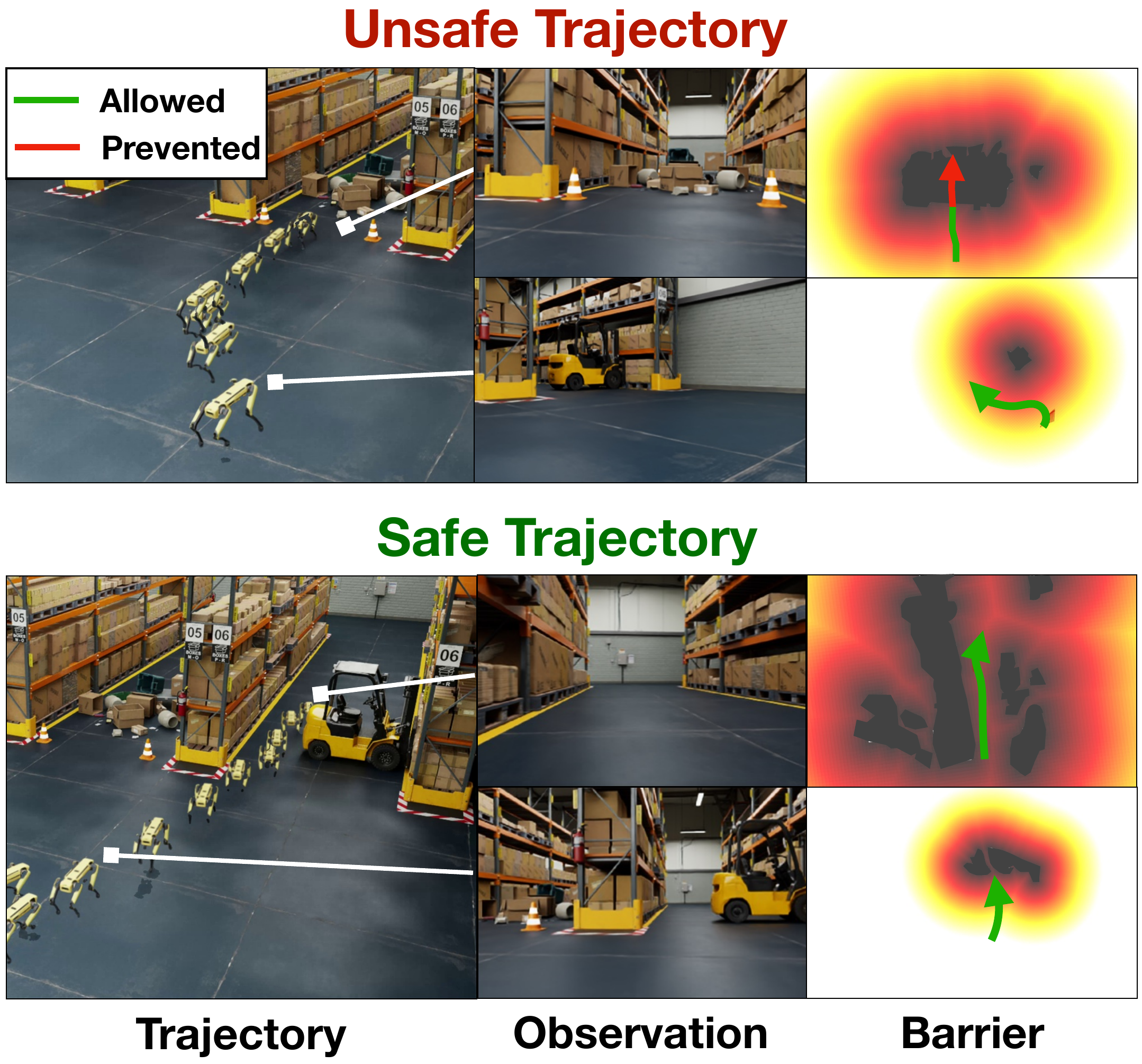}
    \caption{Example safe and unsafe tasks. Top: the nominal controller attempts to guide the robot into an area prohibited by cones. Bottom: the nominal controller attempts to guide the robot through a space constrained aisle.}
    \label{fig:trajectories}
    \vspace{8pt}
\end{figure}

\begin{table}[t!]
    \centering
    \begin{tabular}{ccccccc}
    \toprule
        \multirow{3}{*}{Method} & \multicolumn{3}{c}{Task Success} & \multicolumn{3}{c}{Failure Attribution} \\
        \cmidrule(lr){2-4} \cmidrule(lr){5-7}
        & Total & Safe & Unsafe & Ctx. & Grnd. & Enf. \\
        \midrule
         Oracle  & 98.3\% & 100\% & 96.6\% & 0.0\% & 1.7\% & 0.0\% \\
         \texttt{CORE} & 95.0\% &   96.6\% & 93.3\% & 3.3\%\ & 1.7\% & 0.0\% \\
         No Context & 55.0\% &  93.3\% & 16.6\% & 43.3\% & 1.7\% & 0.0\%  \\
         Geometric & 50.0\% & 100\% & 0\% & 100.0\% & 0.0\% & 0.0\%\\
        \bottomrule \\
    \end{tabular}
    \caption{Simulation results over 60 evaluations.}
    \label{tab:sim_results}
    \vspace{-10pt}
\end{table}

\subsection{Simulation Results}
\label{sec:sim_results}
We compare \texttt{CORE} to the baselines described above.
Tab.~\ref{tab:sim_results} reports success rate and identifies failure modes over all trials. Ctx. indicates a failure to operate with correctly contextualized safety rules, Grnd. indicates a semantic grounding failure, and Enf. indicates a failure to enforce appropriately grounded rules.
\texttt{CORE} achieves a success rate of 96.6\% for safe tasks and 93.3\% for unsafe tasks, with failures occurring because of incorrect contextual reasoning (3.3\%) and grounding (1.7\%).
This rivals the Oracle's performance of 100\% and 96.6\% success rate respectively, with all failures occurring because of incorrect grounding (1.7\%).
In contrast, the No Context method achieves a 93.3\% safe tasks success but only prevents 16.6\% of unsafe tasks.
Unsurprisingly, incorrect contextual information was the main cause of error (43.3\%).
Because this baseline could not infer safety rules online, it relied on generic, pre-defined safety rules predicted by its LLM (e.g., avoiding \verb|hazard_indicator| semantics), which are only effective in limited and obvious cases such as prominent hazard signage.
The geometric baseline succeeds in all safe tasks but fails to prevent any contextually unsafe behavior.

An example unsafe task is illustrated in Fig.~\ref{fig:trajectories} (top), where the nominal controller attempts to guide the robot into a restricted area denoted by a line of cones. As the robot approaches the obstacle, \texttt{CORE} identifies the hazard markings, grounds them in a barrier function, and stops at the safety boundary. Fig.~\ref{fig:trajectories} (bottom) illustrates a safe tasks were the nominal controller guides the robot through a constrained hallway. While \texttt{CORE} identifies both metric and semantic  obstacles, it allows for nominal control.
We observe two primary failure modes of \texttt{CORE}.
In the 3.3\% of failures caused by incorrect context, the VLM failed to identify safety constraints at key timesteps. 
While correct context was provided at future predictions, these were not made in time to prevent unsafe behavior.
We also observed grounding error where ghost obstacles due to depth projection noise would be added to the barrier and block the robot from completing safe tasks.

\begin{table}[!]
    \centering
    \begin{tabular}{cccccc}
    \toprule
        \multirow{3}{*}{Size} & \multirow{3}{*}{VLM} & \multicolumn{3}{c}{Safety Prediction Success} & \multirow{3}{*}{Latency} \\ 
        \cmidrule(lr){3-5}
         &  & Total & Safe & Unsafe  &  \\
        \midrule
        Small & Gemma 3 12B & 75.0\% & 83.0 \% & 67.0\% & 2.2s\\
         & Llava 13B & 0.0\% & 0.0\% & 0.0\% & 2.7s \\
        \midrule
            & Gemma 3 27B & 88.0\% & 91.0\%  & 85.0\% & 4.1s \\
        Large & Qwen3-VL 30B & 75.0\% & 77.0\% & 73.0\% & 26.6s\\
         & Llava 34B & 52.5\% & 49.0\% & 56.0\% & 6.0s \\
        \bottomrule \\
    \end{tabular}
    \caption{VLM Safety Constraint Prediction Success}
    \label{tab:safety_reasoning}
    \vspace{-10pt}
\end{table}

\subsection{VLM Safety Reasoning Comparison}
We then compare the ability of different VLM models to perform contextual safety reasoning for producing contextual constraints.
We focus on three open-source options---Gemma3~\cite{gemma3_techreport2025}, Qwen3-VL~\cite{bai2023qwenvl}, and Llava~\cite{liu2023visualllava}---and we consider small  and large model variants, the latter approaching the upper bound of many embedded compute resources.
We evaluate their ability to predict safe and unsafe predicates over a dataset of 100 images sampled from our experimental setup.
We run all models with 4 bit quantization and compare latency on an Nvidia L40.
Tab.~\ref{tab:safety_reasoning} reports all results. Gemma3 12B achieved success rate of 83.3\% and 67.0\% on safe and unsafe predicate prediction respectively, with an average latency of 2.2s. 
Llava 13B failed to produce syntactically correct outputs, generating natural language outputs that omit predicates or JSON structure.
Gemma 3 27B achieved a success rates of 91.0\% and 85.0\% with an average latency of 4.1s.
Qwen3-VL 30B employs inference-time reasoning which leads to a significantly higher average latency of 26.6s, but surprisingly achieved an inferior performance of 77.0\% and 73.0\%.
Llava 34B achieved a 49.0\% and 56.0\% success rate with an average latency of 6.0s.
These results indicate that the choice of VLM is critical for effective safety reasoning, with Gemma3 27B being best suited for this task to due its balance of speed and accuracy. While trading accuracy for speed---as the Gemma 3 12B variant does---may be advantageous in some settings, we defer analysis to future work.

\subsection{Structured Reasoning Ablation}
\label{sec:vlm_prediction}
We then evaluate the importance of our structured reasoning framework by ablating two key components from the generation procedure described in \S\ref{sec:safety_reasoning}.
First, we remove safety logic instructions from the VLM---providing instead a short description of the problem definition and available spatial operators (No Safety Logic). 
Second, we remove the chain-of-thought reasoning output from the VLM's generation (No CoT Reasoning).
We evaluate the VLM's ability to predict safety predicates on the dataset defined above.
Tab.~\ref{tab:vlm_ablation} indicates that both components contribute to prediction accuracy.
Removing the safety logic instructions reduces performance to 90.0\% and 80.0\%, suggesting the safety logic particularly helps the VLM identify safety hazards.
Removing the VLM's CoT reasoning more severely degraded accuracy, reducing performance to 61.0\% and 77.0\%, with the most significant drop coming on safe region detection.
While this variant achieved a lower latency by producing fewer tokens per generation, this improved latency came at a high accuracy cost.
These results suggest that structured safety reasoning---both via a defined safety logic and explicit reasoning in the generation---are essential for reliable safety prediction.

\begin{table}[t]
    \centering
    \begin{tabular}{cccc}
    \toprule
         Prompt & Safe & Unsafe  & Latency \\
        \toprule
        Full (ours)  & 91.0\%  & 85.0\% & 4.1s \\
        No Safety Logic & 90.0\% & 80.0\% & 4.1s \\
        No CoT Reasoning & 61.0\% & 77.0\% & 2.8s \\
        \bottomrule \\
    \end{tabular}
    \caption{Safety Constraint Prediction Ablation}
    \vspace{-8pt}
    \label{tab:vlm_ablation}
\end{table}

\subsection{Hardware Experiments}

We then apply \texttt{CORE} on a Boston Dynamics Spot Robot operating in a lab space and we compare to the geometric baseline.
We evaluate 3 safe and 3 unsafe tasks, each repeated 5 times, for a total of 30 tasks per method.
Tab.~\ref{tab:hardware_results} reports results using the same metrics as \S\ref{sec:sim_results}.
\texttt{CORE} achieves a success rate of 86.6\% split evenly across safe and unsafe tasks, while the geometric method allows for all safe tasks while preventing no unsafe tasks.
The increased failure rate as compared to simulation results (Tab.~\ref{tab:hardware_results}) came primarily from grounding error. 
In one instance, the VLM predicted an overly conservative specification and prevented safe traversal. 
All other failures were caused by incorrectly localized safety constraints and misdetected semantics.

\subsection{Safe Traversal Guarantees}
We generate perception uncertainty safe traversal bounds using the method of \S\ref{the_perception_uncertainty}. The detection function is calibrated from the Gemma 3 27B results (see Tab. \ref{tab:safety_reasoning}) within the VLM safety reasoning comparison in \S\ref{sec:safety_reasoning}. We target a safe traversal guarantee $\delta^*$, assigning to $\gamma$ one third of the total allowed risk. The safe initial radius $R$ is set in the experiments to $4$ meters. To apply the method of \S\ref{the_perception_uncertainty}, the VLM latency and maximal allowed robot navigation speed (i.e., in the planner) are ``nondimensionalized'' into a fundamental unit linking information and safety: the number of observations per distance-$R$ traversal. From this, we lower bound the safe traversal rate.\footnote{Additional optimization details and discussion of modeling assumptions are deferred to the Appendix.} 
We then search for and find a certificate at $\delta = 0.1$, giving a guarantee that 90\% of trajectories will be safely traversed, subject to our calibration of the measurement function being safe. These results are empirically validated in Table \ref{tab:sim_results}, in which \texttt{CORE} successfully identifies semantically unsafe regions and avoids them with empirical success of approximately 93\%.

\begin{table}[]
    \centering
    \begin{tabular}{ccccccc}
    \toprule
        \multirow{3}{*}{Method} & \multicolumn{3}{c}{Task Type} & \multicolumn{3}{c}{Failure Attribution} \\
        \cmidrule(lr){2-4} \cmidrule(lr){5-7}
        & Total & Safe & Unsafe & Ctx. & Grnd. & Enf. \\
        \midrule
         \texttt{CORE} & 86.6\% & 86.6\% & 86.6\% & 3.9\% & 9.0\% & 0.0\% \\
         Geometric & 50.0\% & 100\% & 0.0\% & 100.0\% & 0.0\% & 0.0\%\\
        \bottomrule \\
    \end{tabular}
    \caption{Hardware results over 30 evaluations}
    \vspace{-8pt}
    \label{tab:hardware_results}
\end{table}

\section{Limitations and Future Work}
\label{sec:limitations}

We identify several limitations and avenues for future work.
The first set of limitations regard \texttt{CORE}'s use of VLMs.
Our perception uncertainty model assumes an average error rate but does not consider per-frame prediction uncertainty.
Developing uncertainty-aware VLMs is an outstanding research topic; future work may both propose methods for this and use uncertainty-aware perception to inform safe control. 
Follow-up work may also leverage distillation techniques to reduce VLM inference time~\cite{ravichandran_prism}.
The second set of limitations regard \texttt{CORE}'s safety formalisms.
While \texttt{CORE} enforces contextual safety via spatially-definecited safe sets, extensions may consider temporally changing environments via the use of signal temporal logic or similar specifications~\cite{donze2010robust}.
\texttt{CORE} further restricts itself to control affine systems. Addressing scenarios where higher order dynamics are crucial---such as high-speed driving---requires higher-fidelity models with more complex control solutions. 
Finally, the \texttt{CORE} framework makes minimal assumptions about the robot's planning and control layer. Follow up work may propose coupled safety solutions that use \texttt{CORE}'s contextual safety reasoning to inform and adapt the robot's nominal plan. This application would be especially interesting given the emergence of language-driven planners for complex mission specifications~\cite{ravichandran_spine, momallm24}

\section{Conclusion}
\label{sec:conclusion}
We consider the problem of enforcing contextual safety for robots in unknown environments.
We address this challenge via CORE, a three-stage contextual safety framework.
CORE uses a VLM to reason about contextual safety constraints given visual observations, grounds these constraints into the robot's environment, then enforces them via a CBF-based controller with probabilistic safety guarantees.
We validate CORE through simulation and hardware experiments, demonstrating that CORE rivals an oracle with ground truth context while significantly outperforming methods without contextual reasoning. We also present ablation studies that demonstrate the importance of structured contextual safety reasoning and validate CORE's probabilistic safety guarantees.

\bibliographystyle{IEEEtran}
\bibliography{references}

\newpage

% \appendix

\clearpage
\appendices

\renewcommand{\thelstlisting}{A.\arabic{lstlisting}} % Prepend "A." to listing numbers
\setcounter{lstlisting}{0}                          % Reset listing counter

\renewcommand{\thefigure}{A.\arabic{figure}} % Prepend "A" to figure numbers
\renewcommand{\thetable}{A.\Roman{table}}  % Prepend "A" to table numbers

\setcounter{figure}{0}                     % Reset figure counter
\setcounter{table}{0}

\section*{Summary}

We provide additional details about our method in \S\ref{sec:appendix_method}, including CORE's spatial operators (\S\ref{sec:appendix_ops})  the VLM's prompt (\S\ref{sec:appendix_prompt}), and a proof of Theorem~\ref{the_probabilistic_cbf_theorem}. 
In \S\ref{sec:appendix_experiments}, we provide further details about our implementation (\S\ref{sec:appendix_details}), experimental tasks (\S\ref{sec:appendix_tasks}), provide additional analysis (\S\ref{sec:appendix_analysis}), and describe how we numerically solve Theorem~\ref{the_probabilistic_cbf_theorem} to generate the safe traversal bounds presented in \S\ref{sec:experiments}.

\section{Method}
\label{sec:appendix_method}

We now provide details on CORE's spatial operators, the VLM prompt, and a proof of Theorem~\ref{the_probabilistic_cbf_theorem}.

\subsection{Spatial Operators}
\label{sec:appendix_ops}
We implement CORE with four spatial operators; Fig.~\ref{fig:spatial-operators} illustrates how each operator may be used to construct a safety predicate.

\verb|ON(class)|: Indicates all pixels corresponding to \verb|class|. This operator is reserved for traversable surfaces and is implemented by extracting the pixels corresponding to a segmented class. The safety logic instructs the VLM to identify surfaces intended for navigation, such as sidewalks or the ground in an office building.

\verb|NEAR(class)|: Indicates all pixels corresponding to \verb|class|. This operator is reserved for objects and has an identical implementation to \verb|ON|. The safety logic instructs the VLM to identify objects that pose collision risks.
    
\verb|AROUND(class)|: Indicates the vicinity around \verb|class| and is implemented by dilating the class segmentation mask using a kernel size of 50 pixels. The safety logic instructs the VLM to identify semantics that indicate the vicinity is dangerous or would result in socially unacceptable behavior, such as navigating very close to a pedestrian.

\verb|BETWEEN(class)|: Indicates the area between instances of \verb|class| and is implemented by first finding the pixel sets corresponding to \verb|class| and then computing the convex hull around those sets. The safety logic instructs the VLM to identify semantics that both represent a prohibited area and are arranged such that they clearly restrict an area.

\begin{figure*}
    \centering
    \includegraphics[width=0.95\linewidth]{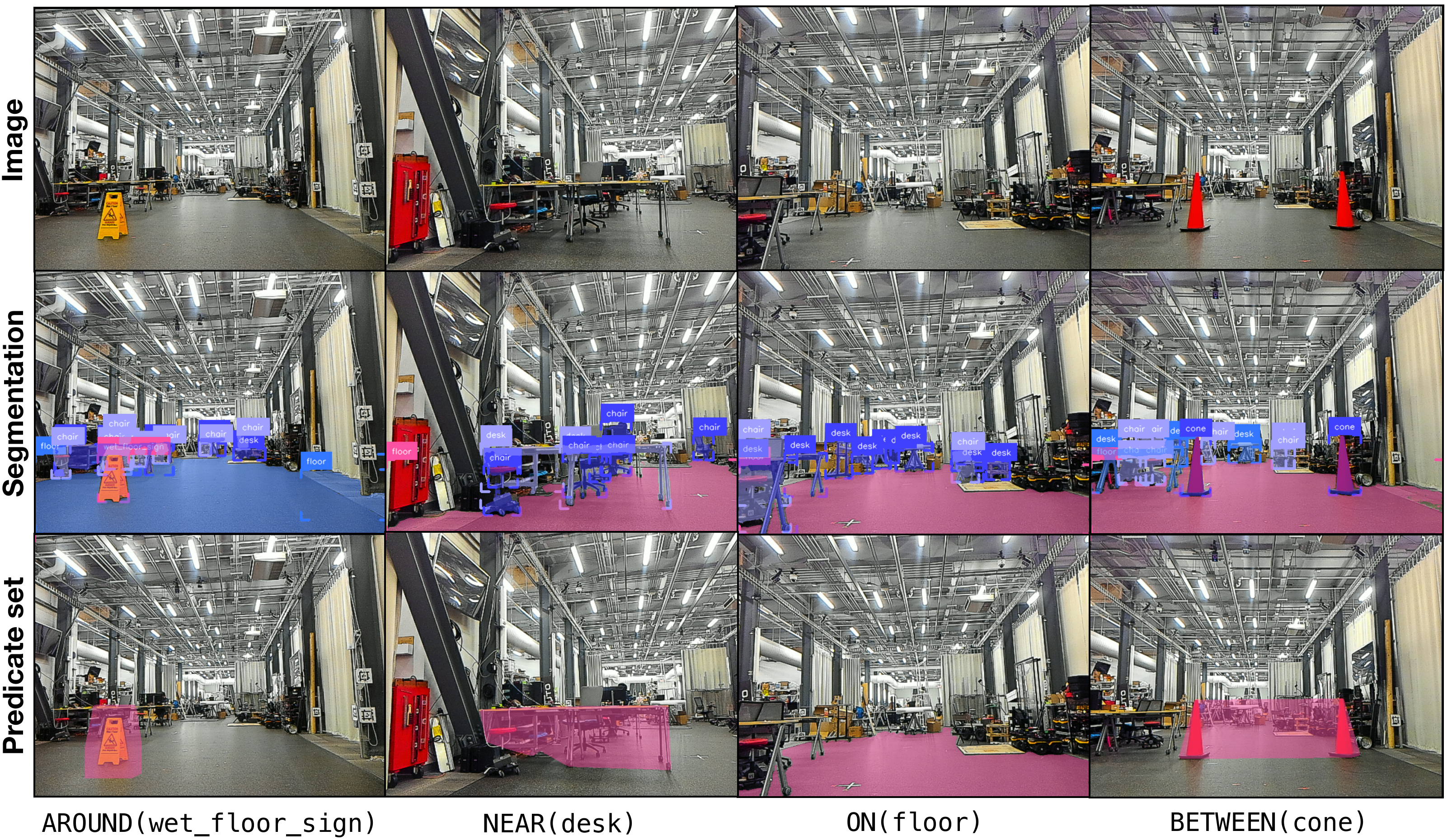}
    \caption{Example safety predicates constructed from the four spatial operators considered   in our work. For each predicate, we show the source observation and segmentation mask. }
    \label{fig:spatial-operators}
\end{figure*}

\subsection{VLM Prompt}
\label{sec:appendix_prompt}
The VLM prompt provides a structured reasoning process for defining contextual safety risks and corresponding constraints without referencing specific semantics in the environment. It is structured into four parts: a role summary, a semantic class definition, a description of spatial operators, and the output format.

The role summary instructs the VLM to perform contextual safety reasoning for a small mobile robot. The semantic class definition outlines five categories of objects and terrain classes relevant for contextual safety reasoning:
\begin{enumerate}
    \item [1.] Metric obstacles: physical obstacles that block the robot's path.
    \item [2.] Hazard indicators: entities that imply the surrounding area is dangerous.
    \item [3.] Socially restricted zones: entities that the robot should not approach closely.
    \item [4.] Semantic barriers: entities that indicate prohibited space.
    \item [5.] Navigable surfaces: surfaces designed and intended for travel.
\end{enumerate}

For each spatial operator, the prompt provides a definition, a description of the logic for applying that operator, and a safety verdict indicating when the operator should be invoked. Finally, the prompt instructs the VLM to provide output in the following JSON format:
\begin{tcolorbox}[colback=gray!3, colframe=black,left=1mm, right=1.5mm, top=1.5mm, bottom=1mm] \small
\begin{minted}[breaklines, breakanywhere]{json}
{
  "safety_logic": "Briefly justify safety decisions (e.g. why AROUND is needed for X)",
  "classes": "class_a, class_b, class_c",
  "unsafe_regions": "NEAR(class_a), AROUND(class_b), BETWEEN(class_c)",
  "safe_regions": "ON(class_d), BETWEEN(class_a)",
}
\end{minted}
\end{tcolorbox}
We provide the full prompt in Listing~\ref{listing:prompt}.

\subsection{Perception Uncertainty}

Here, we give the full proof of Theorem \ref{the_probabilistic_cbf_theorem}, and discuss the operationalization of the result into interpretable design decisions for guaranteeing the closed-loop robot behavior. 
\begin{proof}[Proof of Theorem \ref{the_probabilistic_cbf_theorem}]

    Restating from Sect. \ref{sec:method}: we assume the existence of a known measurement function model $\underline{m}(r)$ that is conservative with respect to $1-\gamma$ of the contextual scenarios which it will encounter, along with an initial safe radius $R > 0$. The intuition of the proof is two-fold: we demonstrate by induction that assimilation of additional measurements monotonically decreases the expectation of the inverse regularized distance. This can be thought of as ensuring the absence of `transient effects' in the data assimilation. Then, we utilize Ville's Inequality and the theory of nonnegative supermartingales (NSMs) to give a `steady-state' condition for safe probabilistic traversal. 

    We first consider the expectation over the initial value of the inverse regularized distance function, for any initial distribution over nearest undetected unsafe regions subject to the safe initial radius constraint. By the monotonicity of $d^{-1}(\cdot)$, it can easily be verified that $r \geq R > 0 \text{ w.p. }1 \iff 0 < d^{-1}(r;c, \ell) \leq \frac{c}{R + \ell} \text{ w.p. } 1$. For a bounded, nonnegative random variable, the expectation (pre-measurement) is bounded by the supremum of the support: 
    \begin{equation*}
        \mathbb{E}[d^{-1}(r; c, \ell)_{t=0^-}] \leq \frac{c}{R + \ell}. 
    \end{equation*}
    We use this initial discussion to emphasize that though the system is assumed certain to start in a non-intersecting configuration with an unsafe region, it may not (pre-measurement) satisfy the conditions of Eq. \ref{the_theorem_equation}, which are more stringent than a safe initial pose. 
    
    Now, we make a natural and non-restrictive simplifying assumption, which enforces the intuition that any physical sensor has a limited sensing range. 
    \begin{assumption}[Bounded Maximal Sensing Radius]
        \label{bounded_radius_assumption}
        There exists $D \gg R > 0$ such that, for all $r \geq D$, $\underline{m}(r) = 0$. 
    \end{assumption}
    This is non-restrictive in that, even if there exists an $\underline{m}^*(r)$ for which it is not true, the choice 
    \begin{equation*}
        \underline{m}^*_D(r) := 1[r < D]\underline{m}^*(r)
    \end{equation*}
    is, by definition, a conservative approximation. Similarly with the following technical assumption:
    \begin{assumption}[Non-Boundary Detection Function]
        \label{asymptotic_radius_assumption}
        For all $r < D$, $0 < \underline{m}_D(r) < 1$. 
    \end{assumption}
    In principle, we do not allow for the detection probability to ever equal $1$, though it may be allowed arbitrarily close (especially as $r \downarrow 0$). This is again a conservative approximation which numerically regularizes subsequent analysis. Similarly, if $\underline{m}_D(r) = 0$ for some $r' < D$, one may again conservatively approximate the true detection by setting $D \gets r'$. 
    Now, we truncate the domain of Eq. \ref{expected_inverse_distance_versus_r_equation} to $D$ using Assumption \ref{bounded_radius_assumption}, and show that the upper bound of the expected inverse distance grows monotonically (towards $D$) as $k$ increases. Using the Leibniz rule and noting that the bounds are independent of the variable $k$, we observe that, for any prior distribution over current radius to undetected obstacles, we have:
    \begin{equation}
        \begin{split}
            & \frac{d}{dk} \mathbb{E}[d^{-1}(r; c ,\ell)\{k\}] = \frac{d}{dk}\int_{0}^D \frac{c(1-\underline{m}_D(r))^k}{(r+\ell)Z(k, \underline{m}_D)}p_0(r)dr \\
            & = \int_0^D \frac{c(1-\underline{m}_D)^{k} (Z\ln{(1-\underline{m}_D(r)) - \nabla_kZ})}{(r+\ell)Z^2}p_0(r) dr
        \end{split}
    \end{equation}
    Noting that $Z$ is defined precisely to ensure that the expectation is taken with respect to a probability measure, it is apparent that 
    \begin{equation*}
        Z(k, \underline{m}_D(r)) := \int_{0}^D (1-\underline{m}_D(r))^k p_0(r)dr, 
    \end{equation*}
    and as a consequence:
    \begin{equation*}
        \nabla_k Z(k, \underline{m}_D(r)) = \int_0^D \ln{(1-\underline{m}_D(r))}(1-\underline{m}_D(r))^{k}p_0(r)dr.
    \end{equation*}
    This allows for useful simplification: 
    \begin{equation}
        \frac{\nabla_k Z}{Z} = \mathbb{E}[\ln{(1-\underline{m}_D(r))}].
    \end{equation}
    \begin{assumption}[Monotonicity of $\underline{m}_D(r)$]
    \label{monotonicity_assumption}
        We assume (again, non-restrictively) that $\frac{d}{dr} \underline{m}_D(r) < 0$.
    \end{assumption}
    This is again non-restrictive, because for any $\underline{m}_D(r)$ which does not satisfy this property, one can always construct a conservative approximation which satisfies it by taking a pointwise minimum. Separating the preceding equation, we have: 
    \begin{equation*}
        \begin{split}
            & \frac{d}{dk} \mathbb{E}[d^{-1}(r; c ,\ell)\{k\}] = \frac{d}{dk}\int_{0}^D \frac{c(1-\underline{m}_D(r))^k}{(r+\ell)Z(k, \underline{m}_D)}p_0(r)dr \\
            & = \int_0^D \frac{c\ln{(1-\underline{m}_D(r))(1-\underline{m}_D(r))^kp_0(r)}}{(r+\ell)Z}dr + \ldots \\
            & \ldots + \int_0^D \frac{-c \mathbb{E}[\ln{(1-\underline{m}_D(r))}](1-\underline{m}_D(r))^kp_0(r)}{(r+\ell)Z}dr \\
            & = \mathbb{E}\bigg[\big(\frac{c}{r+\ell}\big)\big(\ln{(1-\underline{m}_D(r))}\big) \bigg] - \ldots \\
            & \hspace{10mm} \ldots - \mathbb{E}\bigg[\big(\frac{c}{r+\ell}\big)\mathbb{E}\big[\ln{(1-\underline{m}_D(r))}\big] \bigg] \\
            & =  \mathbb{E}\bigg[\big(\frac{c}{r+\ell}\big)\big(\ln{(1-\underline{m}_D(r))}\big) \bigg] - \ldots \\
            & \hspace{10mm} \ldots - \mathbb{E}\bigg[\big(\frac{c}{r+\ell}\big)\bigg]\mathbb{E}\bigg[\ln{(1-\underline{m}_D(r))}\bigg] \\
            & = \text{Cov}(F(r), G(r)) \\
            & \leq 0.
        \end{split}
    \end{equation*}
    We have denoted in this simplification $F(r) = \frac{c}{r+\ell}$ and $G(r) = \ln{(1-\underline{m}_D(r))}$. The last line follows from the definition of covariance for monotone increasing functions and the identity that $\text{Cov}(-F, G) = -\text{Cov}(F, G)$. Specifically, we have that $G(r)$ and $-F(r)$ are (strictly) monotone increasing in $r$, and therefore have (positive) nonnegative covariance. From this, it follows directly to get the desired result. 

    Having demonstrated monotonic decrease as $\{k\}$ (the number of measurements taken at a given position) increases, we now consider a controller which effectuates the rule specified in Eq. \ref{the_theorem_equation}. Choosing this $k$ appropriately, we consider the stochastic process $\{D_t\}$ defined as 
    \begin{equation}
        D_t = \frac{c}{r^*_t + \ell},
    \end{equation}
    where $r^*_t$ denotes the true current distance to the nearest undetected obstacle. Note that this represents the true, `in-the-world' realization of undetected obstacles, whereas Eq. \ref{the_theorem_equation} considers the expectation of $D_t$. By construction, $D_t$ is nonnegative with probability $1$. Therefore, we can use Ville's Inequality to bound the right-quantiles of $D_t$ with respect to its expectation. In particular, for a general stochastic process of the form $\{X_t\}_{t \geq 0}$, the inequality states that, if the process expectation is contracting with respect to the filtration -- that is, $\forall t \text{ }\mathbb{E}[D_t \rvert \mathcal{F}_{t-1}] \leq D_{t-1}$-- then:\footnote{We do not necessarily guarantee this in the case that $D_t < c(\delta-\gamma)/\ell$. Under the controllers defined in the optimization, this is satisfied. In some cases, though, those controllers are conservative, as one may allow $\mathbb{E}[D_t]$ to grow w.r.t. $D_{t-1}$ if $D_{t-1} \ll c(\delta-\gamma)/\ell$. In that case, this generalizes to what is known as an $E$-process, which also satisfies Ville's Inequality.} 
    \begin{equation*}
        \mathbb{P}\bigg[\exists t \in \mathbb{N} : X_t \geq \frac{X_0}{\alpha} \bigg] \leq \alpha \text{, } \forall \alpha \in (0, 1).
    \end{equation*}
    Thus, applying the inequality to the structure of $D_t$ but reasoning in reverse (i.e., in \emph{shaping} the process), we observe that enforcing the NSM property with expectation held at $c(\delta-\gamma)/\ell$ implies a strong guarantee:
    \begin{equation*}
        \mathbb{E}[D_t] \leq \frac{c(\delta-\gamma)}{\ell} \text{ } \forall t \implies \mathbb{P}[\exists t : D_t \geq \frac{c(\delta-\gamma)}{\ell (\delta-\gamma)}] \leq (\delta -\gamma).  
    \end{equation*}
    The condition $D_t \geq c/\ell$ is precisely the condition of entering an unsafe region. Therefore, the probability of entering an unsafe region is bounded at probability $\delta - \gamma$, for the $1-\gamma$ fraction of cases in which the measurement function is valid. 

    A union bound with the $\gamma$ error probability of the measurement function completes the result. 
\end{proof}

\section{Experiments}
\label{sec:appendix_experiments}

We now provide further implementation details, a description of experimental tasks, experimental analysis, and discuss how we derive safe traversal bounds from Theorem~\ref{the_probabilistic_cbf_theorem}.

\subsection{Implementation Details}
\label{sec:appendix_details}
This subsection provides implementation details not covered in the main text. We run the semantic grounding and contextual safety enforcement modules on an AMD Ryzen Threadripper 3960X with an Nvidia RTX 6000 GPU, and deploy VLMs on an Nvidia L40 GPU. The Contextual Safety Reasoning module uses the Ollama framework for model serving\footnote{\url{https://ollama.com/}}. We report the memory requirements of each model in Table~\ref{tab:vlm_memory}. 
Interesting, although Llava 34B has the most parameters, Qwen3-VL 30B has the largest memory consumption. 

The Semantic Grounding module uses a pinhole camera model for image projection. The simulated camera has focal length and principal point parameters $f_x = 274.9$, $f_y=376.7$, $c_x = 250.0$, $c_y = 160.0$. For real-world experiments, we use Spot's gripper camera, which has parameters $f_x=552.0$, $f_y=552.0$, $c_x=320.0$, $c_y=240.0$. Given pixel coordinates $u, v$ and corresponding depth $d$, we compute the 3D point relative to the camera origin $(x, y, z)$ as follows:
\begin{align}
    x = \frac{d (u - c_x)}{f_x}, \quad y = \frac{d (v - c_y)}{f_y}, \quad z = d.
\end{align}

Obstacles close to the robot obstruct its field of view, preventing the VLM from receiving full scene context and increasing the likelihood of incorrect safety predictions. We therefore discard measurements within 3 meters of the robot. We also discard measurements beyond 7 meters, as we found depth estimates degraded past this range. Future work may address this limitation by fusing images from multiple sensors or using wide-angle lenses.

We report the latency of CORE's modules in Table~\ref{tab:module_latency}. The latency of the Semantic Grounding module depends on the number of predicates it must enforce. As the number of required predicates varies across contexts, we observe a relatively high standard deviation (1.14s).

\begin{table}[]
    \centering
    \begin{tabular}{cc}
    \toprule
    Module & Latency (s) \\
    \midrule
        Safety Reasoning & 4.37 $\pm$ 0.45  \\
        Semantic Grounding  &  0.75 $\pm$ 1.14 \\
        Safety Enforcement & 0.04 $\pm$ 0.15  \\ 
        \bottomrule \\
    \end{tabular}
    \caption{Module latency mean and standard deviation}
    \label{tab:module_latency}
\end{table}

\begin{table}[]
    \centering
    \begin{tabular}{cc}
    \toprule
    Module & Memory (GB) \\
    \midrule
        Gemma 3 12B & 10.4 \\
        Llava 13B & 11.3 \\
        Gemma 3 27B & 20.4 \\ 
        Qwen3-VL 30B & 24.3 \\ 
        Llava 34B & 21.4\\
        \bottomrule \\
    \end{tabular}
    \caption{VLM memory requiremnets}
    \label{tab:vlm_memory}
\end{table}

\subsection{Tasks}
\label{sec:appendix_tasks}

We evaluate CORE on twelve tasks, six designed to be unsafe and six designed to be safe. Each task was repeated five times with randomized starting positions.

\begin{enumerate}
    \item[1.] (Unsafe): The controller attempts to guide the robot into a prohibited area in a warehouse as indicated by a line of cones.
    \item[2.] (Unsafe): The controller attempts to guide the robot within 1 meter of an operating forklift in a warehouse.
    \item[3.] (Unsafe): The controller attempts to guide the robot into a potentially dangerous area indicated by a wet floor sign in a hospital.
    \item[4.] (Unsafe): The controller attempts to guide the robot close to a person in a hospital.
    \item[5.] (Unsafe): The controller attempts to guide the robot from a road to a sidewalk through a curb, bypassing the sidewalk ramp.
    \item[6.] (Unsafe): While navigating on a sidewalk, the controller attempts to cut through a grassy area.
    \item[7.] (Safe): The controller guides the robot through open space in a warehouse.
    \item[8.] (Safe): The controller guides the robot through a space-constrained aisle.
    \item[9.] (Safe): The controller guides the robot through a hospital hallway.
    \item[10.] (Safe): The controller guides the robot through a hospital waiting area.
    \item[11.] (Safe): The controller guides the robot through a shopping mall parking lot.
    \item[12.] (Safe): The controller guides the robot along a sidewalk in a residential area.
\end{enumerate}

\subsection{Additional Analysis}
\label{sec:appendix_analysis}
We provide additional details about nominal behavior and observed failure modes. As discussed in \S\ref{sec:experiments}, CORE fails when the VLM does not provide appropriate safety constraints. This occurred during one trial of Task 3, where the VLM correctly identified the trash can and table as metric obstacles but failed to identify the wet floor sign, despite this semantic hazard being visible in the VLM's input image. The output is shown below and illustrated in Fig.~\ref{fig:vlm_failiure_mode} (right):
\begin{tcolorbox}[colback=gray!3, colframe=black,left=1mm, right=1.5mm, top=1.5mm, bottom=1mm] \small
\begin{minted}[breaklines, breakanywhere]{text}
safety_logic: The floor is traversable. The trash can and table are solid obstacles requiring a collision buffer. The arrangement of the trash can and table does not create a barrier, so no BETWEEN is needed. AROUND is not needed as there are no hazards or social entities.,
classes: "floor, trash_can, table,
unsafe_regions: NEAR(trash_can), NEAR(table),
safe_regions: ON(floor)
\end{minted}
\end{tcolorbox}
The VLM behaved nominally during the other four trials, providing the correct constraints. An example output is shown below and illustrated in Fig.~\ref{fig:vlm_failiure_mode} (left):
% The VLM behaved nominally during the other four trials, providing the correct constraints. And example output is below and illustrated in Fig.~\ref{fig:vlm_failiure_mode} (right):
\begin{tcolorbox}[colback=gray!3, colframe=black,left=1mm, right=1.5mm, top=1.5mm, bottom=1mm] \small
\begin{minted}[breaklines, breakanywhere]{text}
safety_logic: The vending machine, table, and wet floor sign are all obstacles. The wet floor sign indicates a hazard in its vicinity, requiring an AROUND buffer. The floor is traversable.,
classes: vending_machine, wet_floor_sign, table, floor,
unsafe_regions: NEAR(vending_machine), NEAR(table), NEAR(wet_floor_sign), AROUND(wet_floor_sign),
safe_regions: "ON(floor)
\end{minted}
\end{tcolorbox}

\begin{figure}
    \centering
    \includegraphics[width=0.95\linewidth]{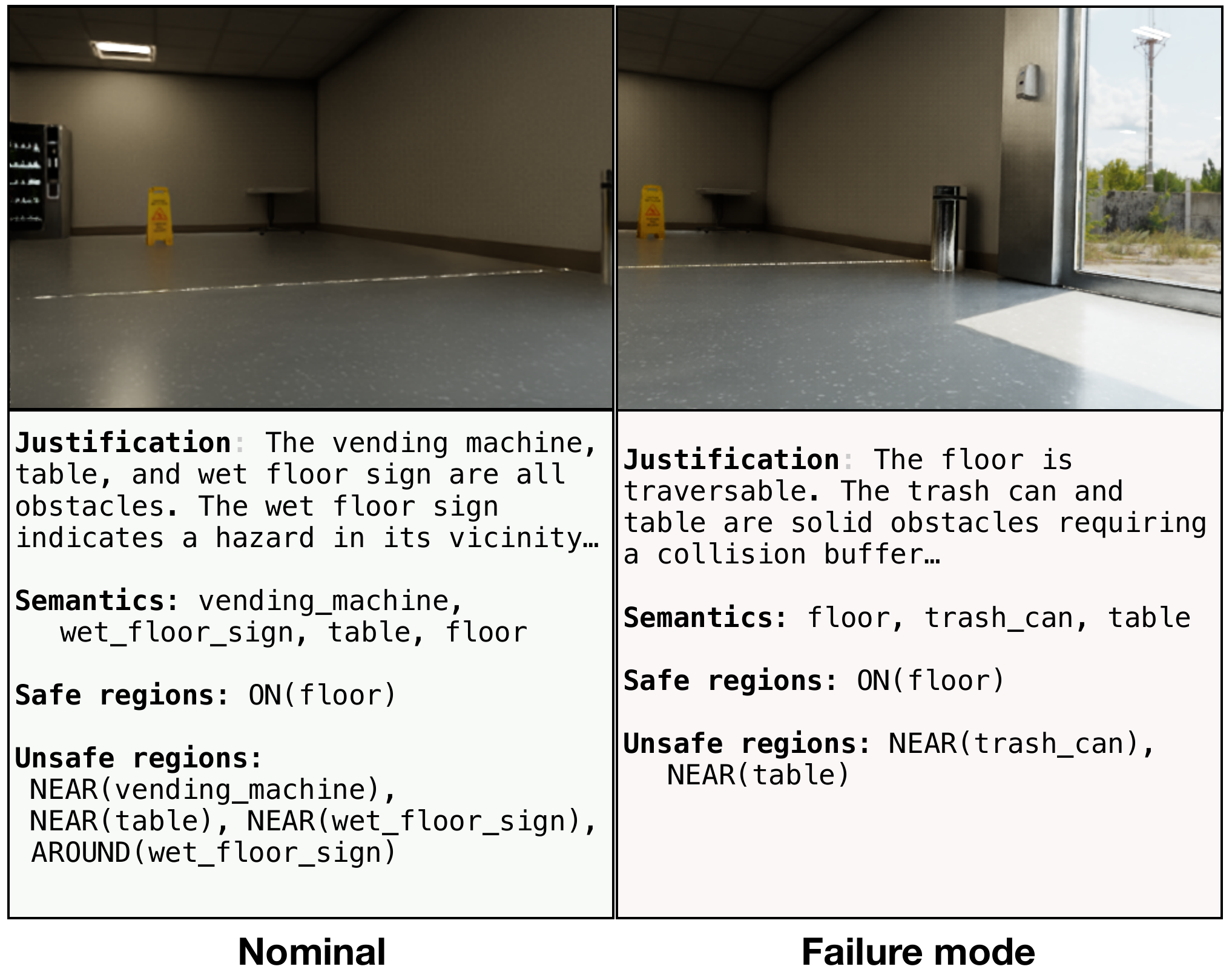}
    \caption{Example nominal behavior and failure mode. The VLM's input image is shown on top, with its response below. Left: failure mode where the VLM does not identify the wet floor sign, so the appropriate safety constraint is never grounded or enforced. Right: nominal behavior where the VLM correctly identifies the wet floor sign as a semantic hazard.}
    \label{fig:vlm_failiure_mode}
\end{figure}

\subsection{Perception Uncertainty Optimization Problem}
\label{sec:appendix_solving_perception_uncertainty}
We now discuss the problem of selecting parameters $c, \ell$ in order to achieve a bound at level $\delta - \gamma$. The key intuition is the non-dimensionalization of the physical space into the unit of `measurements before collision.' There are two different ways to operationalize this guarantee: the method posed in Theorem \ref{the_probabilistic_cbf_theorem} (which results in unnatural `stop-start' motion), or by considering the measurement function as determining a maximal safe speed. We present both, noting that the latter is the more practical. 

\subsubsection{Number of Observations Per Movement}
Here, we consider worst-case changes in the expected regularized inverse distance. We assume a known (perhaps conservative) model of $\underline{m}_D$. The resulting optimization problem is: 
\begin{equation*}
\label{perception_uncertainty_optimization}
    \begin{split}
        \kappa^*_{NOPM} & = \hspace{2mm} \min_{(k, c, \ell) \in \mathbb{N} \times \mathbb{R}_+^2} k \\
        \text{s.t. } & \int_{0}^{D} \frac{c (1-\underline{m}_D(r))^k}{(r+\ell)Z(k, \underline{m}_D(r))}p_0^*(r)dr \leq \frac{c(\delta - \gamma)}{\ell}.
    \end{split}
\end{equation*}
The interpretation of this guarantee (which is independent of time $t$) is that a policy that samples $\kappa^*_{NOPM}$ times independently per movement period is robust to a worst-case change in state ($p_0^*(r)$) over that period (which depends on the speed of the robot and latency of perception). The worst-case change in state is, intuitively, a `beeline' towards an obstacle at varying realized distances. That this enforces the supermartingale / E-process property follows directly from Sect. \ref{sec:method}. 
\subsubsection{Maximum Traversal Speed}
To optimize $k^*_t$, we must consider a worst-case instance of an obstacle at distance $D$, towards which the motion planner is assumed to move directly at maximum speed. Then, we need to solve the following optimization (can be done via binary search):
\begin{equation*}
\label{perception_uncertainty_optimization}
    \begin{split}
        \kappa^*_{MTS} & = \hspace{2mm} \min_{(k, c, \ell) \in \mathbb{N} \times \mathbb{R}_+^2} k \\
        \text{s.t. } \sum_{i=1}^{k-1} & \frac{c}{\ell + \frac{D(k-i)}{k})} \prod_{j=0}^{i-1} (1 - \underline{m}_D(\frac{D(k-j)}{k}))) \cdot \underline{m}_D(\frac{D(k-i)}{k})) \\
        & \hspace{10mm} \ldots + \frac{\mathcal{R}(\underline{m}_D)c}{\ell} \leq \frac{c(\delta - \gamma)}{\ell}.
    \end{split}
\end{equation*}
The implication is that the robot is then able to safely traverse the space at up to maximal speed $D / \kappa^*_{MTS}$. Intuitively, the optimization decomposes detection into discrete instances of data assimilation (i.e., when new measurements are obtained). The initial terms are the realized inverse regularized distances for the `beeline' scenario; the product sums encode the sequence of Bernoulli detection outcomes that result in detecting at exactly the $i^{th}$ step. $\mathcal{R}(\cdot)$ is a remainder term, interpretable as the probability of directly entering the unsafe region. By nonnegativity, a lower bound on the achievable $\delta - \gamma$ certificate is precisely $\mathcal{R}(\cdot)$. 

\subsubsection{Getting a Certificate}
We use the more natural second formulation to derive a certificate. To do this, we take the chosen navigation speed parameter and realized VLM latency constraints to fix a maximal feasible $\kappa^*_{MTS}$, here equal to $3$, based on $D=4$ meters, a maximum speed of $0.35$ms$^{-1}$, and a latency of approximately $3$ seconds. 

To choose a measurement function, we take a high-probability lower bound of the empirical unsafe detection calibration (nominally $85$\%); this bound is found to be $75$\%. Here, we make the (not conservative) assumption that the lower bound holds independently of the context (i.e., $\gamma = 0$); this can of course be robustified to a trivial certificate of $75$\% safe traversal (corresponding to the module failing $25$\% of the time and working perfectly -- that is, $\underline{m}_D(r) \approx 1[r < D]$ -- in the remaining contexts). This would amount to $\delta = \gamma = 0.25$. In our setting, we add an $\epsilon$-decrease term to enforce all of the assumptions, resulting in $\underline{m}_D(r) = 0.75 - (\epsilon r)/D$. 

The free variables are two-dimensional, and the problem can be solved with any blackbox optimization scheme. We are more interested in attaining specific levels of confidence, so we search for a certificate to $\delta = 0.1$ and $\gamma=0$. This is attained at $c = 1$, $\ell = 0.1$. The residual for this solution is small, suggesting low suboptimality.

\begin{listing*}
\label{listing:prompt}
\begin{tcolorbox}[colback=gray!3, colframe=black,left=1mm, right=1.5mm, top=1.5mm, bottom=1mm] \small
\begin{minted}[breaklines, breakanywhere]{text}
<instructions>
You are the Vision-Language Navigation Module for a mobile robot Your goal is to output safe planning data by analyzing the image for **Metric Obstacles**, **Semantic Barriers**, and **Environmental Hazards**.

### 1. SEMANTIC CLASS DEFINITION
Identify Objects in these categories:
- **Metric Obstacles:** physical entities that block space.
- **Hazard Indicators:** Objects that imply the *surrounding area* is dangerous.
- **Socially Restricted Zones:** Any entity (animate or inanimate) that commands a "buffer of respect."
- **Semantic Barriers:** Objects arranged to signal "Do Not Enter."
- **Navigable Surfaces:** Surfaces suitable or intended for travel.

### 2. SPATIAL OPERATOR RULES
You must output regions using these operators. ONLY list an operator if it is relevant for immeditate navigation and NOT in the distant horizon.
Follow the specific logic for each:

**A. The "Collision" Rule (for NEAR)**
- **NEAR(class):** Represents the immediate physical collision buffer.
- **Logic:** ALL solid objects are collision risks.
- **Verdict:** `NEAR` regions for all solid objects are **UNSAFE**.

**B. The "Buffer" Rule (for AROUND)**
- **AROUND(class):** Represents a semantic danger or social etiquette zone.
- **Logic:** 1. For **Hazard Indicators**: Use this for entities where the *vicinity* is dangerous, even if the robot doesn't touch the entity itself.
    2. For **Socially Restricted Zones**: The robot must not enter personal or otherwise socially unacceptable space.
- **Verdict:** `AROUND` regions for hazards and social entities are **UNSAFE**.

**C. The "Grouped Barrier" Rule (for BETWEEN)**
- **Between(class)** represents a prohibited area. Usage: only ONE class as argument.
- **Logic:** This rule takes precedence over AROUND when multiple hazard indicators form a pattern.
    - **UNSAFE:** If multiple hazard indicators are arranged in a **line, curve, or perimeter** to block a path. You must mark the gap `BETWEEN` them as a hazard.
    - **SAFE:** If objects are scattered without a blocking pattern, or if the gap is an intended portal.
- **Verdict:** Evaluate the **arrangement**, not just the object type.

**D. The "Surface" Rule (for ON)**
- **Logic:** Evaluate the functional intent of the surface.
- **Verdict:** **SAFE** only for **Navigable Surfafces**.
- **Verdict:** **UNSAFE** for **Non-Navigable Surfaces** or surface hazards.

### 3. OUTPUT FORMAT
Output a **Single Valid JSON Object**.

**STRICT FORMATTING RULES:**
- Return a flat JSON object.
- The values for `unsafe_regions` and `safe_regions` must be **Single Strings** containing comma-separated operators.
- **DO NOT** use arrays `[]` or nested objects `{}` inside the values.
- **CRITICAL:** Populate `"safety_logic"` first.

**Correct JSON Structure:**
{
  "safety_logic": "Briefly justify safety decisions (e.g. why AROUND is needed for X)",
  "classes": "class_a, class_b, class_c",
  "unsafe_regions": "NEAR(class_a), AROUND(class_b), BETWEEN(class_c)",
  "safe_regions": "ON(class_d), BETWEEN(class_a)",
}
</instructions>
\end{minted}
\end{tcolorbox}
\caption{VLM prompt outlining safety reasoning instructions}
\end{listing*}

\end{document}